\documentclass{article}
\usepackage{Ari}

\usepackage{microtype}
\usepackage{graphicx}
\usepackage{subfigure}
\usepackage{booktabs} % for professional tables

\usepackage{hyperref}

\usepackage[accepted]{icml2024}

\usepackage{amsmath}
\usepackage{amssymb}
\usepackage{mathtools}
\usepackage{amsthm}

\usepackage[capitalize,noabbrev]{cleveref}

\theoremstyle{plain}
\newtheorem{theorem}{Theorem}[section]

\newtheorem{lemma}[theorem]{Lemma}

\theoremstyle{definition}
\newtheorem{definition}[theorem]{Definition}

\theoremstyle{remark}

\usepackage[textsize=tiny]{todonotes}

\icmltitlerunning{On Stronger Computational Separations Between Multimodal and Unimodal Machine Learning}

\begin{document}

\twocolumn[
\icmltitle{On Stronger Computational Separations Between\\ Multimodal and Unimodal Machine Learning}



\icmlsetsymbol{equal}{*}

\begin{icmlauthorlist}
\icmlauthor{Ari Karchmer}{yyy}
% \icmlauthor{Firstname2 Lastname2}{equal,yyy,comp}
% \icmlauthor{Firstname3 Lastname3}{comp}
\end{icmlauthorlist}

\icmlaffiliation{yyy}{Department of Computer Science, Boston University, Boston, MA, USA}

\icmlcorrespondingauthor{Ari Karchmer}{arika@bu.edu}

\icmlkeywords{Multimodal machine learning, theory of machine learning, average-case complexity, ICML}

\vskip 0.3in
]



\printAffiliationsAndNotice{} % otherwise use the standard text.

\begin{abstract}

Recently, multimodal machine learning has enjoyed huge empirical success (e.g. GPT-4).
Motivated to develop theoretical justification for this empirical success, Lu (NeurIPS '23, ALT '24) introduces a theory of multimodal learning, and considers possible \textit{separations} between theoretical models of multimodal and unimodal learning.
In particular, Lu (ALT '24) shows a computational separation, which is relevant to \textit{worst-case} instances of the learning task.
In this paper, we give a stronger \textit{average-case} computational separation, where for ``typical'' instances of the learning task, unimodal learning is computationally hard, but multimodal learning is easy. 
We then question how ``natural'' the average-case separation is. Would it be encountered in practice? To this end, we prove that under basic conditions, any given computational separation between average-case unimodal and multimodal learning tasks implies a corresponding cryptographic key agreement protocol. 
We suggest to interpret this as evidence that very strong \textit{computational} advantages of multimodal learning may arise \textit{infrequently} in practice, since they exist only for the ``pathological'' case of inherently cryptographic distributions. However, this does not apply to possible (super-polynomial) \textit{statistical} advantages.

\end{abstract}

\section{Introduction}
\label{submission}

For humans, multimodal perception---the ability to interpret the same or similar information expressed in multiple ways (e.g. text and image)---is absolutely critical to learning. We hold it as self-evident that access to multiple representations of the same idea can ease the process of forming a mental model applicable to new situations (``when you put it that way...'').

Empirical triumphs of \textit{Machine Learning} from multimodal data such as GPT-4 \citep{achiam2023gpt}, Gemini \citep{team2023gemini}, and Gato \citep{reed2022generalist} suggest that multimodal perception is also really useful for some machine learning tasks. \citet{lu2023theory, lu2023computational} introduces a formal study of multimodal versus unimodal machine learning tasks, in order to develop theoretical justification for the empirical results (see also \citet{huang2021makes} and others; we elaborate on related work in Section \ref{section:related}). However, the theory of multimodal learning is still in its infancy. The main theoretical question is:
\begin{quote}
    \emph{Is multimodal data truly (provably) more useful than unimodal data, or is it a mirage?}
\end{quote}

To attack this question, \citet{lu2023theory} first shows a \textit{statistical} separation: that there exist machine learning tasks that do require asymptotically more samples to complete when the data is expressed unimodally as opposed to multimodally. Second, \citet{lu2023computational} shows that not only is there a statistical separation, but there also exist machine learning tasks that might be \textit{computationally} easier when given access to bimodal data (two modes), rather than just unimodal data. The computational separation of \citet{lu2023computational} identifies a machine learning task that is possible in polynomial time with bimodal data, but not with unimodal data, \textit{for it's worst-case instance}. This means that the unimodal learning task could still possibly be easy on most or ``typical'' instances. 
Of course, Lu's separation requires a relatively weak assumption of computational hardness: that a certain ${\sf NP}$-hard problem is not also in ${\sf P}$. 

In this work, we continue to develop a theory of multimodal learning, in pursuit of the truth about how useful multimodal data is (when compared to unimodal data). In particular, we study the existence of stronger computational separations, which apply to the \textit{average-case} instances of learning tasks. An average-case computational separation can inform the practice of multimodal learning more comprehensively than a worst-case separation, since it would apply with high probability over a randomized process that determines the learning task.

More specifically (though still informally), an average-case computational separation is a multimodal learning task, where a ``typical'' instance of the task is learnable in polynomial time, while the corresponding ``typical'' instance of a \textit{unimodal} learning task is unlearnable in polynomial time. The notion of a typical instance is formalized by considering a fixed \textit{distribution} over learning tasks (and thus some small probability of failure to learn), instead of a universal quantification. We define average-case multimodal learning tasks and computational separations precisely in Section 2.

Our first result is an average-case computational separation, under the computational assumption that learning parities in the presence of a little random learning noise is hard. More specifically, 

\begin{definition}[LPN assumption]\label{intro:def:searchLPN}
For any length parameter $n \in \mathbb{N}$ and noise rate $\theta \in (0,0.5)$, the $t{\rm -LPN}_{\theta, n}$ assumption is that for every probabilistic algorithm $\mbf{I}$ running in time $t(n)$, 
\[
\Pr_{\mbf{x},\mbf{A}, \mbf{b}}[\mbf{I}(\mbf{A}, \mbf{x}\mbf{A} + \mbf{b})= \mbf{x}]< 1/t(n)
\]
Here, $\mbf{x}$ is a uniformly random element of $\ints_2^{1 \times n}$, $\mbf{A}$ is a uniformly random element of $\ints_2^{n \times t(n)}$ and $\mbf{b} \in \ints_2^{1 \times t(n)}$ is sampled element-wise from $\mathrm{Ber}(\theta)$.
\end{definition}

We construct an average-case computational separation under the $\poly{\rm -LPN}_{\theta, n}$ assumption where $\theta \triangleq n^{-0.5}$. 
We refer informally to this assumption as low-noise LPN.\footnote{The low-noise LPN assumption is a popular conjecture in the cryptographic literature, as it is known to imply public key encryption \cite{alekhnovich2003more}, and pseudo-random functions with extremely low circuit depth \cite{yu2016pseudorandom}. In particular, it is common to conjecture subexponential (i.e., $2^{n^{\epsilon}}$) hardness. 
% On the other hand, the ``standard'' LPN assumption, where the noise rate $\theta \in \Theta(1)$, is not known to imply either. Standard LPN does, however, imply the standard cryptographic primitives such as pseudorandom generators and digital signatures.
}

\begin{theorem}[Informal]\label{intro:septhm}
    Under the low-noise LPN assumption, there exists an average-case bimodal learning task that can be completed in polynomial time, and a corresponding average-case unimodal learning task that cannot be completed in polynomial time.
\end{theorem}

\noindent For simplicity, we prove an average-case computational separation between bimodal and unimodal learning. In the context of a separation, this only strengthens the result, as any separation involving bimodal data applies to the multimodal versus unimodal setting.

Low-noise LPN is a relatively natural hardness of learning assumption. However, the bimodal learning task (and corresponding unimodal task) that we construct to prove the separation is pathologically constructed given the assumption (in Section \ref{sketches} we present a sketch of the idea). Therefore, it makes sense to ask: do there exist more natural bimodal learning tasks that constitute average-case computational separations? Indeed, this question was left open by \cite{lu2023computational} even in the context of worst-case separations.

Towards an answer, we look to find the \textit{minimal} computational hardness needed to construct an average-case computational separation between bimodal and unimodal learning. In doing so, we hope to reveal the core computational problem at the center of a separation between multimodal and unimodal learning, so that we can then understand whether it might be frequently encountered in practice.

In this vein, our second result says that to obtain \textit{any} average-case computational separation, we \textit{must} assume enough computational hardness to construct \textit{cryptographic key agreement} protocols. In fact, we construct an explicit key agreement protocol based on any given hypothesized average-case computational separation.\footnote{This is a so-called ``win-win'' result, which may be of independent interest: either secure key agreement protocols exist, or typical instances of unimodal learning tasks can be learned without significantly more computation than the multimodal task.}

\begin{theorem}[Informal]\label{intro:informalmain}
    For any given average-case bimodal learning task that can be completed in polynomial time, such that the corresponding unimodal task cannot be completed in polynomial time, there exists a corresponding cryptographic key agreement protocol.
\end{theorem}

% $\alpha \phi$

\noindent Cryptographic key agreement (KA)---where two parties communicate over an authenticated but insecure channel in order to jointly agree on a \textit{secret} key---is one of the fundamental tasks of cryptography.\footnote{For example, KA is fundamental to  the Transport Layer Security (TLS) protocol for facilitating secure communication over the internet.}  
The existence of KA protocols is known to be equivalent to the existence of public key encryption schemes, and other exotic cryptographic primitives (see e.g. \cite{impagliazzo1995personal} for more information). 
% Thus, Theorem \ref{intro:informalmain} proves that average-case computational separations are only possible when many exotic cryptographic primitives also exist.

\paragraph{Interpreting Theorem \ref{intro:informalmain}.} To apply Theorem \ref{intro:informalmain} to the question of whether their exist more natural average-case computational separations, we suggest the following perspective. Although Theorem \ref{intro:septhm} gives good evidence that super-polymomial average-case computational separations do \textit{exist}, Theorem \ref{intro:informalmain} shows that \textit{any} such separation may not be very ``natural,'' since it needs to be sufficiently ``cryptographic.'' That is, it can be directly used to construct exotic cryptographic primitives. Arguably, ``cryptographic'' data distributions rarely come up in practice, where data is generated from natural processes instead of the precise design of a cryptographer. Hence, we suggest that \textbf{super-polynomial} \textit{computational} advantages of multimodal learning may arise infrequently in practice. Our interpretation seems to contradict the results of practical studies, however our interpretation does not apply to \textit{statistical} advantages of multimodal learning (i.e., less data needed). In fact, even the separation from Theorem \ref{intro:septhm} does not hold in the statistical regime, since LPN can be solved in $2^{o(n)}$ time even with $\poly(n)$ samples \cite{lyubashevsky2005parity}. This would explain why multimodal learning continues to succeed in practice: the advantages are typically statistical, not computational.
% See section \ref{section:conclusion} for a continued discussion.

\paragraph{On polynomial separations.}
Since Theorem \ref{intro:informalmain} derives a cryptographic KA protocol from any given \textit{super}-polynomial separation, we suggest that super-polynomial separations are infrequent in practice. However, polynomial separations (e.g. quadratic computational advantage) may still be relevant in certain practical settings despite not being as totally debilitating as super-polynomial separations. Indeed, a polynomial separation does not necessarily imply (by Theorem \ref{intro:informalmain}) a KA protocol with cryptographic security (i.e., security against all polynomial time adversaries). Therefore, we (conservatively) refrain from suggesting that polynomial separations are unlikely to occur in practice. That being said, \textbf{our proof is general enough to show that any given polynomial separation does still imply a corresponding polynomial-security KA protocol.}\footnote{See the proof of Theorem \ref{theorem:securityBA} for details.} Therefore, a more aggressive interpretation could argue that even large polynomial separations (e.g. a quartic computational advantage) are pretty unlikely to occur in practice.

\paragraph{On low-noise LPN.}
    Theorem \ref{intro:informalmain} also provides a hint for why we do not achieve a separation as in Theorem \ref{intro:septhm} by using the weaker \textit{standard} LPN assumption, where the noise rate $\theta$ can be taken to be any constant fraction less than one half, and the secret parity is uniformly random. Indeed, the standard LPN assumption is not known to imply any form of KA, unlike the low-noise variant. Constructing KA from the standard LPN assumption is a major open problem in the theory of cryptography.

%%% Full version %%%

% Additionally, we remark that, by viewing Theorem \ref{intro:informalmain} in the contrapositive, we obtain a very natural algorithmic benefit of living in \textit{Minicrypt}, Impagliazzo's world where some cryptographic primitives exist (e.g. zero-knowledge proofs, digital signatures), but more advanced \textit{Cryptomania} primitives like KA do not exist. The algorithmic benefit is that average-case learning tasks are not necessarily harder with unimodal data than they are with multimodal data. At the moment, relatively few algorithmic benefits of life in \textit{Minicrypt} are known at all.

%%%

\subsection{Our Constructions: the Main Ideas}\label{sketches}

% \paragraph{Average-case computational separation.}

\paragraph{Theorem \ref{intro:septhm}.} Loosely speaking, the main idea behind the construction of our average-case computational separation is to use the LPN assumption to obtain a very strong \textit{heterogeneity} between the available data modalities. Heterogeneity is a notion concerning multimodal data studied by \cite{lu2023computational}, which he identifies as a fundamental aspect of multimodal learning.

Intuitively, the heterogeneity property (in bimodal learning) is that the two modalities somehow complement each other, so that seeing data from both modalities is significantly better than from just one. Indeed, if the two modalities are very similar, then adding data from the second modality is redundant (as an extreme case, consider when datapoints are sampled identically across all modalities). 
% Hence, the heterogeneity property implies that a useful second modality should contain information that is somehow hard to derive from the first modality (even given labels of that data). 

To construct a \textit{computational} separation, it is clear that the useful second modality should contain information that is \textit{hard to compute} given information in the first modality. Thus, our main idea is to use the trapdoor properties of the LPN assumption to construct a distribution over a modality $\cY \subseteq \real^n$, such that samples from that distribution hide all information about the corresponding sample from a modality $\cX \subseteq \real^n$ (with respect to efficient computation). In cryptographic terms, the mapping from $\cX$ to $\cY$ is \textit{distributionally one-way}. Since the mapping is one-way, there still exists a learnable \textit{connection} (a mapping) from $\cY$ to $\cX$. This one-way connection is an instance of the \textit{connection} property identified by \citet{lu2023computational} as essential to the existence of an advantage from learning with multimodal data.

For our separation, we define the joint data distribution over the two modalities so that the first modality consists of a low-noise LPN instance, and the useful second modality is the parity function that underlies the LPN instance. Obviously, given the LPN assumption, the useful second modality is hard to compute given samples from first. This gives a strong heterogeneity property in a formally justified way.

Simultaneously, we must define the joint distribution over the two modalities so that they can be used to actually learn from the labelled data. Our method to handle this involves injecting hidden data into the first modality which are only recovered given the second modality. Our method is inspired by the ideas behind the low-noise LPN-based public key encryption schemes of \cite{alekhnovich2003more}, and uses key ideas from the Covert Learning algorithm for noisy parities of \cite{canetti2021covert}.

\paragraph{Theorem \ref{intro:informalmain}.} 

We have described how an average-case computational separation needs some form of \textit{distributional one-wayness} between the two modalities. This is required because otherwise either modality could be efficiently sampled given the other, making a reduction from unimodal to bimodal learning feasible.

Distributional one-wayness is a standard notion in cryptography, equivalent to the more fundamental notion of one-wayness \citep{impagliazzo1990no}. In proving Theorem \ref{intro:informalmain}, we show that cryptographic key agreement, a cryptographic primitive thought to be (much) stronger than one-way functions, must be an essential aspect of constructing an average-case computational separation.

Our construction of KA exploits the fact that data sampled from the unimodal task is hard to learn from, unless one also has the the corresponding data from the second modality. Towards KA, it suffices to construct a \textit{bit agreement} protocol, before invoking standard techniques from cryptography to obtain a KA protocol for long keys (see Section \ref{apx:BA} for more information). The first player in the bit agreement protocol (called ``Alice'') uses the average-case computational separation to sample unlabeled bimodal data, and then send only the unimodal data to the second player (called ``Bob''). Bob picks a uniformly random bit $b_B$ (which is the bit he wants Alice to agree with), and if $b_B = 1$, labels the data by a concept sampled according to the multimodal learning task and sends it back, and if $b_B = 0$ responds with uniformly random labels. Alice, given the data labels, applies the multimodal learning algorithm to obtain a hypothesis function. Roughly speaking, Alice can decode Bob's bit with good probability because the accuracy of the hypothesis function resulting from her execution of the multimodal learning algorithm is only good when $b_B = 1$. Finally, the protocol can be shown secure because any polynomial time adversary who, given a view of the protocol, can predict $b_B$ with any probability significantly better than 1/2, can be used to contradict the assumption that the unimodal learning task is hard in polynomial time.

Our bit agreement protocol uses similar ideas to the protocol of \citet{pietrzak2008weak}, who use the existence of a so-called secret-coin weak pseudorandom function to construct KA. However, our analysis is significantly more involved than that of \citet{pietrzak2008weak} since we need to show that an adversary for the protocol implies a unimodal learning algorithm rather than a successful weak pseudorandom function adversary, which is weaker.

% \paragraph{Remark.}
% We note that constructing a KA protocol, as we do, from a worst-case computational separation like that of \cite{lu2023computational} is \textbf{not} possible, since KA requires randomization of the messages between the players.

\subsection{Related Work}\label{section:related}

Few theoretical results on multimodal learning are known at the moment. For those that exist, they consider certain limited scenarios. For example, works of \citet{yuhas1989integration, sridharan2008information, amini2009learning, federici2020learning} consider a situation of learning from multimodal data, but where learning is still possible from each mode individually. This is the so-called \textit{multi-view} setting, and does not produce theoretical justification for any computational separation afforded by access to multimodal data.

On another note, \cite{huang2021makes} study advantages in generalization when learning common latent representations of multimodal data, but not predictors.
Additionally, other works like \citet{yang2015auxiliary} and \citet{sun2020tcgm} make strong distributional assumptions about the data that is sampled from multiple modalities. It is not clear whether those assumptions hold in practice.

Empirically, deep learning from multimodal data has had great success, for example in learning massive general agents like GPT-4 \citep{achiam2023gpt}, Gemini \citep{team2023gemini} and Gato \citep{reed2022generalist}. Also, deep learning for modality \textit{generation} has worked well (e.g. \citet{reed2016generative} for text to image). It is often observed that ML models derived from multimodal data perform better than even fine-tuned models derived from unimodal data. 

\section{Technical Overview}\label{section:tech}
Before proving Theorem \ref{intro:septhm} and \ref{intro:informalmain} in the next sections, we begin by introducing the model for (average-case) multimodal learning, as well as the notions of computational separations.

\subsection{Bimodal Learning}

We follow the model for bimodal learning of \cite{lu2023computational}. In a formal bimodal learning task, two modalities, denoted by $\cX, \cY \subseteq \real^n$, and a label space $\cZ$, form the basis for the selection of datapoints $(x,y,z) \in \cX \times \cY \times \cZ$. In the bimodal PAC-learning task, selection of a dataset consisting of $m$ datapoints abides by a data distribution $\rho$ over $\cX \times \cY \times \cZ$. For an accuracy and confidence $\epsilon, \delta \in (0,1)$, the goal of a PAC-learning algorithm $A$ is to process this dataset generated by $\rho$ so as to generate a hypothesis function $h: \cY \rightarrow \cZ$ that achieves population risk below $\epsilon$, \textit{on the unimodal task of labelling elements of $\cY$ with labels of $\cZ$} (and without loss of generality, labelling elements of $\cX$):
\[
\ell_{\rm pop}(h)\triangleq \Ex{(x,y,z) \sim \rho}{\ell(h,y, z)}\le \epsilon\]
with probability at least $1-\delta$, for some loss function $\ell$. For example, $\ell_{0-1}(h, y, z) \triangleq \mbf{1}[h(y) \not= z]$ (0-1 loss).

In Section \ref{section:separation}, we also consider $\ell_0(h, y, z) \triangleq \frac{1}{|z|}|\{i : \mbf{1}[h(y)_i \not= z_i]\}|$ when $z$ is not a single bit.
The algorithm $A$ is considered efficient if it runs in polynomial time in the parameters $1/\epsilon, 1/\delta$ and $n$.

\subsection{Relationship Between Modalities} \cite{ lu2023computational} defined the bimodal PAC-learning task so that there must exist a (unknown) bijection between $\cX$ and $\cY$ defined for any $(x,y,z)$ in the support of $\rho$. In this work we generalize this so that there only exists a (probabilistic) mapping between $x$ and $y$---we consider this a more practical assumption since many correspondences arising in bimodal learning in practice do not have bijective (or even functional) relationships. For example, consider that a single caption may have many images that it describes, and a single image may have many captions that describe it. 
Furthermore, for any caption, the paired image sampled by the data distribution may be chosen at random from the set of possible images defined by the mapping. 

\paragraph{Probabilistic mappings.} Formally, we represent a unidirectional probabilistic mapping from a set $S$ to a set $T$ by a function $\phi: S \rightarrow [0,1]^{|T|}$, subject to the constraint that the sum over all $s \in S$ of $\phi(s) =1$.
The mapping defined by the function $\phi$ maps an element $s$ to element $t_i \in T$ with probability $\phi(s)_i$. We write $\phi[s]$ to denote a sample from $T$ according to the distribution $\phi(s)$. Frequently, we will  define a probabilistic mapping $\phi$ by explicitly defining the distribution $\phi[s]$ for all $s \in S$, as it is conceptually easier to define.

\subsection{Unimodal Learning} When $\cX, \cY, \cZ$ are clear from the context, we identify a bimodal PAC-learning problem by $\rho$. Arising from a bimodal PAC-learning problem $\rho$ are unimodal PAC-learning problems $\rho_{\cX,\cZ}$ and $\rho_{\cY,\cZ}$. Here, $\rho_{\cX,\cZ}$ and $\rho_{\cY,\cZ}$ denote the distribution $\rho$ over $\cX \times \cY \times \cZ$ projected to $\cX \times \cZ$ and $\cY \times \cZ$ respectively. In unimodal PAC-learning, the task is defined analagously to the bimodal task: the goal of the learning algorithm $A$ for $\rho_{\cY,\cZ}$ (w.l.o.g.) is to produce a hypothesis $h$ such that
\[
\ell_{\rm pop}(h)\triangleq \Ex{(y,z) \sim \rho_{\cY,\cZ}}{\ell(h,y)}\le \epsilon
\]
with probability greater than $1-\delta$ for some loss function $\ell$.

\subsection{Average-case Bimodal Learning}

In this work, we will primarily consider an average-case notion of bimodal and unimodal learning. Let $\Delta(S)$ denote the convex polytope over all distributions over a set $S$. In the average-case notion of bimodal learning, we assume that the bimodal learning task is sampled according to a meta-distribution $\mu$ over $\Delta(\cX \times \cY \times \cZ)$. This is consistent with the ``Bayesian view'' of the PAC-learning task, where the learner is assumed to have some prior over the possible data distributions. When $\cX \times \cY \times \cZ$ is clear from the context, an average-case bimodal learning problem is identified with $\mu$.

More specifically, we consider the meta-distribution $\mu$ to be of the following natural form. Let $\chi$ be a fixed distribution over the first modality $\cX$. Let $\eta$ be a distribution over a set of probabilistic mappings $\phi: \cX \rightarrow [0,1]^{|\cY|}$ that transform elements of the first modality $\cX$ to elements of the second modality $\cY$. Finally, let $\zeta$ be a distribution over a set of probabilistic mappings $\psi: \cY \rightarrow [0,1]^{|\cZ|}$. The meta-distribution $\mu$ selects a data distribution $\rho$ by sampling $\phi \sim \eta$ and $\psi \sim \zeta$; the data distribution $\rho$ thus samples a datapoint by sampling $x \sim \chi$, and returning $(x, \phi[x], \psi[\phi[x]])$. We write $\mu = (\chi, \eta, \zeta)$ to be explicit about such average-case bimodal learning tasks.

% Specifically, we will consider a natural format for the meta-distribution $\mu$. We assume the fact that $\mu$ is a product distribution $(\mu_1, \mu_2)$, where $\mu_1$ is a distribution over $\Delta(\cX \times \cY)$, $\mu_2$ is a distribution over $\Delta(\cZ)$, and $\mu$ is sampled by drawing from $\mu_1$ and $\mu_2$ independently.

% Let $\cX = \cY = \{0,1\}^n$ and $\cZ = \{0,1\}$. Consider for example the following distribution $\rho$, which is defined by a matrix $\textbf{A} \in \ints_2^{n \times n}$ and a vector $\textbf{w} \in \ints_2^{n}$.
% A datapoint $(x,y,z)$ is sampled by drawing the unlabeled point $x$ from the first modality uniformly at random over $\cX$. Then the unlabeled datapoint $y$ is sampled from the second modality by computing $y' = \textbf{A}x$, and negating each bit in $y'$ with probability $1/\sqrt{n}$. Finally, the label $z$ is computed by taking 

\subsection{Average-case Computational Separations}

Let us now formalize what is an average-case computational separation in multimodal learning.

\begin{definition}\label{def:avgsep}
    We say that an average-case bimodal learning task $\mu$ (with respect to a loss function $\ell$) is a super-polynomial \textit{computational separation} if it holds that:
\begin{itemize}
    \item There exists a polynomial $p: \nat \rightarrow \nat$ such that there is a time $p(n)$ probabilistic algorithm $A$ such that, when $\rho \sim \mu$, and given access to $p(n)$ datapoints sampled according to $\rho$, $A$ outputs a hypothesis that achieves population risk $\ell_{\rm pop}(h) \le 1/2-1/p(n)$ with probability $1/p(n)$ over $\mu, \rho$ and randomness of $A$.
    \item For $\rho_{\rm uni} \in \{\rho_{\cX,\cZ}, \rho_{\cY,\cZ}\}$: For every polynomial $t: \nat \rightarrow \nat$, and every probabilistic algorithm $A$ running in time $t(n)$, when $\rho \sim \mu$, and $A$ is given access to $t(n)$ datapoints sampled according to $\rho_{\rm uni}$, $A$ outputs a hypothesis such that $\ell_{\rm pop}(h) > 1/2-1/t(n)$ in the unimodal task $\rho_{\rm uni}$ with probability at least $1-1/t(n)$ over $\mu, \rho_{\rm uni}$ and randomness of $A$.
\end{itemize}
\end{definition}

% Note that, one-sidedness comes from the fact that while unimodal PAC-learning $\rho_{\cY,\cZ}$ (sampled according to $\mu$) is hard, we do not require this to be true for $\rho_{\cX,\cZ}$. In the context of constructing a key agreement protocol as we do later, using a one-sided super-polynomial computational separation is a weaker assumption than the analogous definition for a two-sided separation.

We note that the separation only requires population risk $\ell_{\rm pop}(h) \le 1/2-1/p(n)$ for the bimodal case, and $\ell_{\rm pop}(h) > 1/2-1/t(n)$ for the unimodal case (with high probability). Again, this makes our construction of KA a \textit{stronger} result.

On the other hand, when we construct the separation in section \ref{section:separation}, we get a difference in population risk that is optimally large. The learning algorithm for the bimodal task achieves $\ell_{\rm pop}(h) \le n^{-0.5}$ (which is optimal), while we prove hardness of achieving $\ell_{\rm pop}(h) < 1/2-1/t(n)$ for any polynomial 
$t: \nat \rightarrow \nat$ in the unimodal task.

\subsection{Relationship to LUPI}

The model for multimodal PAC learning defined by \citet{lu2023computational}, which is used in this paper, bears resemblance to the Learning Using Privileged Information (LUPI) paradigm of \citet{vapnik2009new}. In fact, for the bimodal case, the two models of learning are the same (we omit formal proof, which follows immediately by definition). 

Both LUPI and bimodal learning consider triplets of information (rather than the standard of pairs). These triplets are given as input to the learning algorithm at training time, while at test time, only pairs are received. In this way, the two models consider situations where multiple modalities are accessible for training machine learning models, and may be effective for learning problems that act on just a single modality at test time. In the case of LUPI, the second modality is motivated by the presence of a ``teacher.'' The ``teacher'' can try to ease the learning process by giving the learner some auxiliary information about the data. In the case of multimodal learning, the additional modalities are motivated by the contemporary success of multimodal perception in AI, where there is an abundance of data but not necessarily a ``teacher.''

Since the LUPI paradigm of \citet{vapnik2009new} is the same as the bimodal learning model of \citet{lu2023computational}, both our main results (Theorem \ref{intro:septhm} and \ref{intro:informalmain}) apply to the LUPI paradigm. Previous work \citep{vapnik2009new, lapin2014learning, vapnik2015learning} on understanding the LUPI paradigm (e.g. proving upper bounds, lower bounds, and separations) focused on statistical learning settings such as empirical risk minimization (ERM). To our knowledge, our results are the first to consider the average-case computational power of the LUPI paradigm.

\section{Average-Case Separation}\label{section:separation}

In this section, we construct a super-polynomial computational separation, assuming hardness of the $t-{\rm LPN}_{n^{-0.5}, n}$ problem. We recall that an average-case multimodal learning problem $\mu = (\chi, \eta, \zeta)$ consists of:
\begin{itemize}
    \item $\chi$: a distribution over the modality $\cX$.
    \item $\eta$: a distribution over probabilistic mappings that transform $\cX \rightarrow \cY$.
    \item $\zeta$: a distribution over probabilistic mappings that transform $\cY \rightarrow \cZ$.
\end{itemize}

In order to construct our separation, we will only need to define $\eta$ so that it places the entire probability mass on a single probabilistic mapping $\phi: \cX \rightarrow [0,1]^{|\cY|}$. However, because we prove a separation (i.e., a ``negative'' result), this only strengthens the result.

\subsection{Construction of Separation}\label{subsection:construction_separation}

Let $\mathrm{Ber}(m)$ denote the Bernoulli random variable with mean $m \in [0,1]$.

Consider the following average-case multimodal learning problem $\mu = (\chi, \eta, \zeta)$. The modalities are $\cX = \ints_2^{1\times n} \times [n]$, $\cY = \ints_2^{n \times n} \times \ints_2^{1\times n}$, and $\cZ = \ints_2^{n \times 1} \times \ints_2$. All sums are computed modulo 2.
\begin{itemize}
    \item $\chi$: Sample uniformly random $i \in [n]$, and $\mbf{x} \in \ints_2^{n\times 1}$ where $\mbf{x}_j$ is sampled i.i.d. from $\mathrm{Ber}(n^{-0.5})$. Output $(\mbf{x}, i)$.
    \item $\eta$: With probability 1, output the probabilistic mapping $\phi: \cX \rightarrow [0,1]^{|\cY|}$. We define 
    \[
    \phi[(\mbf{x}, i)] = (\mbf{A}, \mbf{x}\mbf{A} + \mbf{b} + \mbf{e}^{(i)})
    \]
    where $\mbf{A} \in \ints_2^{n \times n}$ is a uniformly random, $\mbf{b} \in \ints_2^{1\times n}$ is such that $\mbf{b}_i$ is sampled i.i.d. from $\mathrm{Ber}(n^{-0.5})$, and $\mbf{e}^{(i)} \in \ints_2^{1\times n}$ is defined so that $(\mbf{e}^{(i)})_j = 1$ if and only if $j = i$.
    \item $\zeta$: Sample probabilistic mapping $\psi_{\mbf{w}}: \cY \rightarrow [0,1]^{|\cZ|}$ by sampling random vector $\mbf{w} \in \ints_2^{n \times 1}$ such that $\mbf{w}_i$ is sampled i.i.d. from $\mathrm{Ber}(n^{-0.5})$. We define 
    \[
    \psi_{\mbf{w}}[(\mbf{Y}, \mbf{y})] = (\mbf{Y}\mbf{w} + \mbf{b}', \mbf{y}\mbf{w} + \mbf{b}'')
    \]
    where $\mbf{b}' \in \ints_2^{n\times 1}$ is such that $\mbf{b}'_i$ is sampled i.i.d. from $\mathrm{Ber}(n^{-0.5})$. The bit $\mbf{b}''$ is also sampled i.i.d. from $\mathrm{Ber}(n^{-0.5})$.
\end{itemize}

\begin{theorem}[Separation]\label{thm:sep}
    Under the $\poly{\rm -LPN}_{\theta, n}$ assumption for for $\theta \triangleq n^{-0.5}$, the multimodal learning task $\mu = (\chi, \eta, \zeta)$, as defined above, is a super-polynomial computational separation. 
\end{theorem}
\begin{proof}
    The statement follows immediately from Theorem \ref{thm:feasible_bimodal} and Theorem \ref{thm:unfeasible_unimodal}.
\end{proof}
We prove Theorems \ref{thm:feasible_bimodal} and \ref{thm:unfeasible_unimodal} in the following two sections.

\subsection{An Efficient Multimodal Learning Algorithm}

We begin by proving that there exists an efficient algorithm for the multimodal PAC-learning task defined by $\mu$. This is the part of the separation that shows the feasibility of the learning task given bimodal data.

Now, observe that, given samples of the form $(x,y,z) \sim \rho$ for $\rho \sim \mu$, the multimodal PAC-learning task is learned optimally by finding the vector $\mbf{w}$ underlying $\psi_\mbf{w}$. Hence, we now give an algorithm that finds $\mbf{w}$, and then outputs the optimal hypothesis.

\begin{algorithm}
  \caption{$A_\mu \ | \ \rho \sim \mu$}
  \label{A1}
  \begin{algorithmic}[1]
    \STATE \textbf{Input}: $n^3$ samples $(x,y,z) \sim \rho$.\\  \STATE \textbf{Output}: $h: \cX \times \cY \rightarrow \cZ$.

    \STATE Interpret each example $(x_j,y_j,z_j)$ as $((\mbf{x}_j, i_j), (\mbf{Y}_j, \mbf{y}_j), (\mbf{z}_j, z_j))$.

    \STATE Sort all examples $((\mbf{x}_j, i_j), (\mbf{Y}_j, \mbf{y}_j), (\mbf{z}_j, z_j))$ into $n$ bins labelled by $i \in [n]$ by value of $i_j$.

    \FOR{each bin $b_i$}

    \FOR{example $((\mbf{x}_j, i_j), (\mbf{Y}_j, \mbf{y}_j), (\mbf{z}_j, z_j))$ in bin $b_i$:}

    \STATE Compute $\alpha_{i,j} = \mbf{x}_j\mbf{z}_j + z_j$
    \ENDFOR
    \STATE Compute $\mbf{w}'_i$ by taking the majority vote over $\alpha_{i,j}$ for all $j$.

    \ENDFOR
    \STATE \textbf{Output} $h((\mbf{Y}, \mbf{y})) = (\mbf{Y}\mbf{w}', \mbf{y}(\mbf{w}')^{\sf T})$.
  \end{algorithmic}
\end{algorithm}

We prove that with high probability, the algorithm $A_\mu$ outputs $\mbf{w}' = \mbf{w}$, and this minimizes population risk for the PAC-learning task, with respect to $\ell_0$ loss.

\begin{theorem}\label{thm:feasible_bimodal}
    We have that
    \[
    \Pr_{A_\mu, \rho \sim \mu}[\ell_{\rm pop}(h) \le n^{-0.5} : h \leftarrow A_\mu] \ge 1-\exp(-\Omega(n))
    \]
    with respect to $\ell_0$ loss. Moreover, $A_\mu$ runs in time $\poly(n)$.
\end{theorem}

% \[\ell_{\rm pop}(h)\triangleq \Ex{(x,y,z) \sim \rho}{\mbf{1}[h(x,y) \not= z]} \le \epsilon
% \]

\begin{proof}
    The runtime of $A_\mu$ being $\poly(n)$ is immediate.
 
    To show that population risk is small, we need to show that:
    \begin{align*}
    \Pr_{A_\mu, \rho \sim \mu}\left[\Ex{(x,y,z) \sim \rho}{\ell_0(h,y)} \le 1/n^{0.5} : h \leftarrow A_\mu\right] \\ \ge 1-\exp(-\Omega(n))
    \end{align*}
Consider that if $A_\mu$ outputs hypothesis $h$ such that $\mbf{w}' = \mbf{w}$, where $\mbf{w}$ is the vector sampled by $\zeta$, then $h$ satisfies
\[
\Ex{(x,y,z) \sim \rho}{\ell_0(h,y)} \le 1/n^{0.5}
\]
Thus we will prove that $A_\mu$ finds $\mbf{w'} =\mbf{w}$ with probability at least $1-\exp(-\Omega(n))$ over the randomness of $\rho \sim \mu$ and the $n^3$ examples sampled from $\rho$ given as input to $A_\mu$.

To prove this, let us focus on bit $\mbf{w'}_i$, without loss of generality (the following argument is applies to all $i \in [n]$). The bit $\mbf{w'}_i$ is the majority vote of $\alpha_{ij} = \mbf{x}_j\mbf{z}_j + z_j$ for all examples $((\mbf{x}_j, i_j), (\mbf{Y}_j, \mbf{y}_j), (\mbf{z}_j, z_j))$ conditioned on $i_j = i$. Therefore, if 
\begin{align}\label{eq1}
\Pr_{((\mbf{x}_j, i_j), (\mbf{Y}_j, \mbf{y}_j), (\mbf{z}_j, z_j))}\left[\mbf{x}_j\mbf{z}_j + z_j = \mbf{w}_i | i_j = i\right] \\ \ge 1/2+ \Omega(1)
\end{align}
then $(\mbf{w}')_i = \mbf{w}_i$ with probability at least $1-\exp(-\Omega(n))$, given enough voter participation. Leaving aside the issue of number of votes, note that, by a union bound it would follow that $\mbf{w'} = \mbf{w}$ still with probability $1-\exp(-\Omega(n))$ as desired. We now show (\ref{eq1}), and leave the issue of lower bounding the number of votes (i.e., the number of examples in every bucket) for after.

To show (\ref{eq1}), we expand:
\begin{align*}
    \mbf{x}_j\mbf{z}_j + z_j &= \mbf{x}_j(\mbf{A}_j\mbf{w} + \mbf{b}') + (\mbf{x}_j\mbf{A}_j + \mbf{b}+\mbf{e}^{(i)})\mbf{w} + \mbf{b}''\\
    &= \mbf{x}_j(\mbf{A}_j\mbf{w}) + \mbf{x}_j\mbf{b}' + \mbf{x}_j(\mbf{A}_j\mbf{w}) + \mbf{b}\mbf{w}\\
    &\ \ \ \ +\mbf{e}^{(i)}\mbf{w} + \mbf{b}''\\
    &= \mbf{x}_j\mbf{b}' + \mbf{b}\mbf{w} +\mbf{e}^{(i)}\mbf{w} + \mbf{b}''
    % &= \mbf{A}_j\mbf{w}\mbf{x}_j + \mbf{x}_j(\mbf{b}')^{\sf T} + \mbf{A}_j^{\sf T}\mbf{x}_j\mbf{w}^{\sf T} + \mbf{b}\mbf{w}^{\sf T}+\mbf{e}^{(i)}\mbf{w}^{\sf T} + \mbf{b}''\\
    % &= \mbf{x}_j(\mbf{w}^{\sf T}\mbf{A}_j) + \mbf{x}_j(\mbf{b}')^{\sf T} + \mbf{A}_j\mbf{x}_j\mbf{w}^{\sf T} + \mbf{b}\mbf{w}^{\sf T}+\mbf{e}^{(i)}\mbf{w}^{\sf T} + \mbf{b}''\\
\end{align*}
And now we argue that for $n>4$, 
\[
\Pr\left[\mbf{x}_j\mbf{b}' + \mbf{b}\mbf{w} +\mbf{e}^{(i)}\mbf{w} + \mbf{b}'' = \mbf{w}_i\right] \ge 0.515
\]
To see this, observe that $\mbf{e}^{(i)}\mbf{w} = \mbf{w}_i$, so it suffices to show 
\[
\Pr\left[\mbf{x}_j\mbf{b}' + \mbf{b}\mbf{w} + \mbf{b}'' = 0\right] \ge 0.515
\]
Each term forming the sum inside the probability is an independent random variable. Thus, let us lower bound the probability that each of the three terms is 0. For the first term, we have that 
$\mbf{x}_j\mbf{b}' = \sum{i}^{n} (\mbf{x}_j)_i(\mbf{b}')_i$ and for each $i$, $\Pr[(\mbf{x}_j)_i(\mbf{b}')_i = 1]= 1/n$ (by definition of the sampling process, where $(\mbf{x}_j)_i$ and $(\mbf{b}')_i$ are 1 with probability $1/n^{0.5}$). Hence, for $n>4$, for all $i$, $(\mbf{x}_j)_i$ and $(\mbf{b}')_i$ are = 0, with probability at least $0.326$ (direct computation). Therefore, $\Pr[\mbf{x}_j\mbf{b}' = 0] \ge 0.5+ 0.326/2 \ge  0.663$.

The same argument and conclusion holds for the second term. For the third term, we know that $\Pr[\mbf{b}'' = 1] = 1/n^{0.5}$. Therefore, 
\[
\Pr\left[\mbf{x}_j\mbf{b}' = \mbf{b}\mbf{w} = \mbf{b}'' = 0\right] \ge 0.663^2 \cdot 0.9 \ge 0.395
\]
Also, 
\[
\Pr\left[\mbf{x}_j\mbf{b}' = \mbf{b}\mbf{w} = 1 \land \mbf{b}'' = 0\right] \ge 1/e^2 \cdot 0.9 \ge 0.12
\]
So we conclude that 
\[
\Pr\left[\mbf{x}_j\mbf{b}' + \mbf{b}\mbf{w} + \mbf{b}'' = 0\right] \ge 0.515
\]

Thus, if the number of examples in each bin $b_i$ is at least $n$, then by standard application Chernoff bounds (see Lemma \ref{chernoff}), the majority vote over all $\alpha_{i,j}$ used to compute $\mbf{w}'_i$ matches $\mbf{w}_i$, save for an event of exponentially small probability in $n$. Furthermore, another application of Chernoff and union bounds gives that the number of examples in every bin is at least $n$, save for a bad event that occurs with exponentially small probability in $n$. A final union bound concludes that $A_\mu$ finds $\mbf{w'} =\mbf{w}$ with probability at least $1-\exp(-\Omega(n))$ as desired.
\end{proof}
\begin{lemma}[Chernoff Bound, cf. Theorem 2.1 \cite{janson2011random}]\label{chernoff}
    Let $X \sim {\sf Bin}(m,p)$ and $\lambda = m \cdot p$. For any $t \ge 0$,
    \[
    \Pr[|X-\Ex{}{X}| \ge t] \le \exp\left(\frac{-t^2}{2(\lambda + t/3)}\right)
    \]
\end{lemma}

\subsection{Hardness for Unimodal Learning}

Now that we have shown that the bimodal learning problem is learnable in polynomial time, we will show that the corresponding unimodal task, cannot be learned in polynomial time, unless the low-noise LPN assumption does not hold with respect to polynomial time adversaries.

\begin{theorem}\label{thm:unfeasible_unimodal}
    Let $\mu = (\chi, \eta, \zeta)$ be defined as in section \ref{subsection:construction_separation}. Assume that the $\poly{\rm -LPN}_{\theta, n}$ assumption holds for $\theta\triangleq n^{-0.5}$.
    Then, for every polynomial $t: \nat \rightarrow \nat$, and every probabilistic algorithm $A$ running in time $t(n)$, when $\rho \sim \mu$, and $A$ is given access to $t(n)$ datapoints sampled according to $\rho_{\rm uni}\in \{\rho_{\cX,\cZ}, \rho_{\cY,\cZ}\}$, $A$ outputs a hypothesis such that $\ell_{\rm pop}(h) > 1/2-1/t(n)$ in the unimodal task with probability at least $1-1/t(n)$ over $\mu, \rho$ and randomness of $A$.
\end{theorem}

To prove the theorem, we exploit the \textit{decisional} version of the LPN assumption. Informally, the decisional LPN assumption is that for a suitably defined distribution of LPN samples, it is hard to distinguish them from uniformly random bits. Most importantly, the decisional LPN and \textit{search} LPN (definition \ref{intro:def:searchLPN}) are equivalent for the $\poly{\rm -LPN}_{n^{-0.5}, n}$ regime, due to the existence of a polynomial time serach-to-decision reduction (consult \citet{pietrzak2012cryptography} for more information, and Lemma 1 of \citet{katz2010parallel} for the search-to-decision reduction).

Furthermore, \citet{applebaum2009fast} introduce an important variant of the problem, where the secret parity function is sampled from the same distribution as the noise vector. \citet{applebaum2009fast} show that this variant is as hard as when the secret parity function is sampled uniformly at random. We state this variant of the decisional assumption below. 

\begin{definition}[Decisional LPN]\label{def:dlpn}
For any length parameter $n \in \mathbb{N}$ and noise rate $\theta \in (0,0.5)$, the $t{\rm -DLPN}_{\theta, n}$ assumption is that for every probabilistic algorithm $\mbf{D}$ running in time $t(n)$, 
\begin{align*}
\left| \Pr_{\mbf{D}, \mbf{A}, \mbf{x}, \mbf{b}}\left[\mbf{D}(\mbf{A}, \mbf{x}\mbf{A} + \mbf{b}) =1\right] - \Pr_{\mbf{D}, \mbf{u}}\left[\mbf{D}(\mbf{A},\mbf{u}) =1\right] \right| \\ < 1/t(n)
\end{align*}
Here, $\mbf{A}$ is a uniformly random element of $\ints_2^{n \times t(n)}$, and $\mbf{x}, \mbf{b} \in \ints_2^{1 \times t(n)}$ are sampled element-wise from $\mathrm{Ber}(\theta)$, while $\mbf{u}$ is a uniformly random element of $\ints_2^{1 \times t(n)}$.
\end{definition}

\begin{proof}[Proof of Theorem \ref{thm:unfeasible_unimodal}]
See appendix section \ref{apx:proof_unfeasible_unimodal}.
\end{proof}
 
\section{KA from Computational Separations}

In this section, we construct a cryptographic key agreement protocol, given a super-polynomial computational separation for a multimodal \textit{binary classification} task $\mu$. Here, we have modalities $\cX, \cY$, and the label space $\cZ$ is fixed to $\{0,1\}$. We assume that the computational separation is with respect to $\ell_{0-1}$ loss, since $\mu$ is a multimodal binary classification task.

To construct cryptographic key agreement, it is only necessary to construct a bit agreement protocol. A bit agreement protocol is key agreement for a key of one bit, with a small but nontrivial probability of agreement better than 1/2. By standard cryptographic techniques---parallel repetition and privacy amplification---a bit agreement protocol can be converted in to a full-blow key agreement protocol for long keys. Thus, we will construct a bit agreement protocol here. We refer the reader to \cite{holenstein2006immunization} for more information about constructing key agreement from bit agreement. We give a formal definition of bit agreement in the appendix.

\subsection{A Bit Agreement Protocol}

Consider the protocol between players Alice and Bob specified below.

\begin{algorithm}[H]
      \caption{Protocol 1}
      \label{protocol1}
      \begin{algorithmic}[1]
        \STATE Alice samples $x_1, \cdots x_{k+1} \sim \chi$ and $\phi \sim \eta$.
        \STATE Alice computes $y_i = \phi[x_i]$ for all $i \in [k+1]$.
        \STATE Alice sends Bob $(y_i)_{i \in [k+1]} \in \cY^{k+1}$.
        \STATE Bob samples a bit $b_B \in \{0,1\}$.
        \STATE If $b_B = 0$, Bob samples a random string $w \in \cZ^{k+1}$, and sends Alice $w$.
        \STATE If $b_B = 1$, Bob samples $\psi \sim \zeta$, and computes $z_i = \psi[y_i]$ for all $i \in [k+1]$, and sends Alice $(z_i)_{i \in [k+1]} \in \cZ^{k+1}$.
        \STATE Alice interprets the first $t$ bits of the string she received by considering the $i^{th}$ bit a label for the multimodal datapoint $(x_i,y_i)$. Alice runs a multimodal learning algorithm, using the datapoints $(x_i, y_i, z_i)_{i \in [k]}$, for $\mu = (\chi, \eta, \zeta)$, to obtain a hypothesis $h$. 
        \STATE Bob outputs $b_B$. Alice outputs $\mbf{1}[h(y_{k+1}) = z_{k+1}]$.
        
      \end{algorithmic}
    \end{algorithm}
% \paragraph{Protocol 1.}
% \begin{enumerate}
%     \item Alice samples $x_1, \cdots x_{k+1} \sim \chi$ and $\phi \sim \eta$.
%     \item Alice computes $y_i = \phi[x_i]$ for all $i \in [k+1]$.
%     \item Alice sends Bob $(y_i)_{i \in [k+1]} \in \cY^{k+1}$.
%     \item Bob samples a bit $b_B \in \{0,1\}$.
%     \item If $b_B = 0$, Bob samples a random string $w \in \cZ^{k+1}$, and sends Alice $w$.
%     \item If $b_B = 1$, Bob samples $\psi \sim \zeta$, and computes $z_i = \psi[y_i]$ for all $i \in [k+1]$, and sends Alice $(z_i)_{i \in [k+1]} \in \cZ^{k+1}$.
%     \item Alice interprets the first $t$ bits of the string she received by considering the $i^{th}$ bit a label for the multimodal datapoint $(x_i,y_i)$. Alice runs a multimodal learning algorithm, using the datapoints $(x_i, y_i, z_i)_{i \in [k]}$, for $\mu = (\chi, \eta, \zeta)$, to obtain a hypothesis $h$. 
%     \item Bob outputs $b_B$. Alice outputs $\mbf{1}[h(x_{k+1},y_{k+1}) = z_{k+1}]$.
% \end{enumerate}

\begin{theorem}
    Assume that $\chi, \eta$ and $\zeta$ are samplable in time $\poly(n)$.
    If $\mu = (\chi, \eta, \zeta)$ is an average-case super-polynomial computational separation, then there exist a cryptographic key agreement protocol.
\end{theorem}

\begin{proof}
    The statement follows from theorem \ref{thm:BAcorrect} and \ref{theorem:securityBA}, which prove that protocol 1 is a secure and correct bit agreement protocol, and then applying parallel repetitions and privacy amplification.
\end{proof}

% In order to show that this is a bit agreement protocol, we first need to prove that Alice and Bob output the same bit in step 8 with probability at least $1/2+1/q(n)$ for some polynomial $q: \nat \rightarrow \nat$. Also, they need to do this in time $\poly(n)$. Second, we need to prove that any probabilistic polynomial time adversary $D$, given a view of the interaction, cannot predict the bit output by Alice and Bob with probability larger than $1/2+\nu(n)$ for some negligible function $\nu: \nat \rightarrow \nat$.

\begin{theorem}[Correctness of BA protocol]\label{thm:BAcorrect}
    Suppose that $\mu = (\chi, \eta, \zeta)$ is a super-polynomial computational separation. Also, assume that $\chi, \eta$ and $\zeta$ are samplable in time $\poly(n)$.
    Then, there exists a polynomial $q: \nat \rightarrow \nat$ such that Alice and Bob output the same bit with probability at least $1/2 + 1/q(n)$. In other words, 
    \[
    \ex{\mbf{1}[b_B = \mbf{1}[h(y_{k+1}) = z_{k+1}]]} \ge 1/2+ 1/q(n)
    \]
    Additionally, Alice and Bob each run in polynomial time.
\end{theorem}

\begin{proof}

    Conditioning on $b_B=0$, we know that Alice outputs 0 with probability exactly $1/2$, since in this case Bob chose $z_{k+1}$ uniformly at random. 
    
    Conditioning on $b_B=1$, we now know that because $\mu = (\chi, \eta, \zeta)$ is a, average-case super-polynomial computational separation, there exists a polynomial $p$ and a time $p(n)$ probabilistic algorithm $A$ such that, when $\rho \sim \mu$, and given access to datapoints sampled according to $\rho$ (the bimodal data!), $A$ outputs a hypothesis $\ell_{\rm pop}(h) \le 1/2-p(n)$ with probability $1/p(n)$ over $\mu, \rho$ and randomness of $A$. 
    Thus, when Alice runs this bimodal learning algorithm on the dataset $(x_i, y_i, z_i)_{i \in [k]}$, she obtains a hypothesis $h$ such that $\ell_{\rm pop}(h) \le 1/2-p(n)$ with probability at at least $1/p(n)$, and with remaining probability has $\ell_{\rm pop}$ at worst equal to $1/2 + \nu(n)$. If $\ell_{\rm pop}$ was larger than 1/2 + any negligible function of $n$, then her hypothesis could be efficiently tested and negated to obtain one with $\ell_{\rm pop}(h) \le 1/2-q(n)$.
    
    Therefore, using that $\Pr[b_B = 0] = 1/2$, we can conclude that 
    {
    \begin{align*}
        &\ex{\mbf{1}[b_B = \mbf{1}[h(y_{k+1}) = z_{k+1}]]}\\
        &\ge \frac{1}{2}\left(\frac{1/2+1/p(n)}{p(n)} +  \left(\frac{1}{2} - \nu(n)\right)\left(1-\frac{1}{p(n)}\right)\right)\\  &\ \ \ \ + \left(\frac{1}{2}\right)^2 \\
        % &\ge \frac{1}{4} + \frac{1}{4p(n)} + \frac{1}{p(n)^2} + \frac{1}{4} - \frac{1}{4p(n)} - \nu(n)\\
        &\ge \frac{1}{2}+ \frac{1}{p(n)^2} - \nu(n)
    \end{align*}}

    Finally, it is immediate that both Alice and Bob run in polynomial time, since $\chi, \eta$ and $\zeta$ are polynomial time samplable, and due to the fact that by assumption there exists a polynomial time bimodal learning algorithm for $\mu$.
    This suffices to complete the proof of the theorem.
\end{proof}

We now prove security of the protocol; that is, an adversary who views the interaction between Alice and Bob, denoted by ${\rm View}({\rm A} \leftrightarrow {\rm B})$, can predict the bit output by Bob with probability at most $1/2+\nu(n)$ for a negligible function $\nu$.

\begin{theorem}[Security of BA protocol]\label{theorem:securityBA}
    Suppose that $\mu = (\chi, \eta, \zeta)$ is a super-polynomial computational separation.
    Then, for any polynomial $t$, and algorithm $D$ running in time $t(n)$,
    \[
    \ex{D({\rm View}({\rm A} \leftrightarrow {\rm B})) = b_B} < 1/2+ 1/t(n)
    \]
\end{theorem}
\begin{proof}[Proof Sketch.]
% We sketch the argument.
%     Towards contradiction, suppose that there exists a probabilistic time $t(n)$ algorithm $D$ such that 
%     \[
%     \ex{D({\rm View}({\rm A} \leftrightarrow {\rm B})) = b_B} \ge 1/2+ 1/t(n)
%     \]
%     This implies that $\mu = (\chi, \eta, \zeta)$ is not a one-sided super-polynomial computational separation. 
    We prove the theorem by applying old techniques from the theory of pseudorandomness and cryptography. In particular, a standard ``hybrid argument'' \cite{goldwasser1982probabilistic} together with a reduction from learning to next-bit prediction in the spirit of \cite{yao1982theory}. See appendix \ref{apx:security} for a detailed proof.
\end{proof}

\section{Conclusion}\label{section:conclusion}

% In addition to the technical results, the key \textit{conceptual} contribution of this work is an heuristic argument that the advantages of multimodal machine learning observed in practice are typically statistical, not computational. That is, multimodal perception typically allows for good training with less data, though perhaps not significantly less computation. Our argument relies on the fact that we can explicitly construct KA from any average-case super-polynomial computational separation. Thus, average-case \textbf{super-polynomial} computational separations may not arise naturally in real world data. However, KA does not follow from any average-case statistical separation, so strong statistical advantages may still be encountered frequently.

% For future work, we suggest to study the possibility of average-case \textit{polynomial} (e.g. $n^2$) computational separations. This would still be relevant to practice, but fundamentally do not concern cryptographic key agreement protocols, which consider super-polynomial adversaries, and are less likely to arise in natural processes.

In addition to the technical results, the key \textit{conceptual} contribution of this work is an heuristic argument that the advantages of multimodal machine learning observed in practice are typically statistical, not computational. That is, multimodal perception typically allows for good training with less data, though perhaps not significantly less computation. Our argument relies on the fact that we can directly and explicitly construct a KA protocol from any given average-case super-polynomial computational separation. Thus, average-case super-polynomial computational separations may not arise naturally in real world data. However, KA does not follow \textit{solely} from any average-case statistical separation, so strong statistical advantages may still be encountered frequently. KA cannot follow from \textit{solely} from any average-case statistical separation, because KA fundamentally requires a computational advantage, while a statistical separation considers only unbounded computational parties and no computational-statistical gaps.

% There are several ways to add to the theory of multimodal learning. First of all, the theoretical formalism of multimodal PAC-learning that we consider could possibly be improved. The theoretical formalism that we used to study separations is still a relatively idealized setting, and there are potentially some aspects of multimodal learning and perception that might not be captured by the formalism. Pinning down ``the right formalism'' is an active area of research.

\paragraph{Future work.} More work can be done to continue to understand the theoretical foundations of multimodal learning. For example, we show that average-case \textit{super-polynomial} computational separations imply cryptographic KA protocols that have super-polynomial security. We use this to present a heuristic argument that such super-polynomial computationally advantages might be rare in the real world. However, if we only assume a polynomial separation (e.g. quadratically more computation is necessary in unimodal learning), then this would could still be relevant to practice, but fundamentally separate from traditional cryptography (which considers super-polynomial adversaries). \textbf{That being said, our proof is general enough to show that any given polynomial separation does still imply a polynomial-security KA protocol.}
To study the polynomial regime further, we propose to investigate relationships to \textit{fine-grained} public key cryptography \citep{lavigne2019public}.  

\section*{Impact Statement}
This paper presents work whose goal is to advance the field of Machine Learning. There are many potential societal consequences of our work, none which we feel must be specifically highlighted here.

\section*{Acknowledgements}

This work was supported by the DARPA SIEVE program, Agreement Nos. HR00112020021 and HR00112020023. The author thanks Zhou Lu for several helpful conversations, and Sivan Sabato for pointing out a relationship to the LUPI paradigm.

\bibliographystyle{plainnat}
\bibliography{bib.bib}

%%%%%%%%%%%%%%%%%%%%%%%%%%%%%%%%%%%%%%%%%%%%%%%%%%%%%%%%%%%%%%%%%%%%%%%%%%%%%%%
%%%%%%%%%%%%%%%%%%%%%%%%%%%%%%%%%%%%%%%%%%%%%%%%%%%%%%%%%%%%%%%%%%%%%%%%%%%%%%%
% APPENDIX
%%%%%%%%%%%%%%%%%%%%%%%%%%%%%%%%%%%%%%%%%%%%%%%%%%%%%%%%%%%%%%%%%%%%%%%%%%%%%%%
%%%%%%%%%%%%%%%%%%%%%%%%%%%%%%%%%%%%%%%%%%%%%%%%%%%%%%%%%%%%%%%%%%%%%%%%%%%%%%%
\newpage
\appendix
\onecolumn

\section{Bit Agreement}\label{apx:BA}

A bit agreement protocol is communication protocol between two parties, Alice and Bob.
Alice and Bob are allowed to communicate over an authenticated (but insecure) channel. They each begin with a common input $n$ delivered in unary ($1^n$) which constitutes a security parameter.
At the end of their communication, Alice and Bob each output a single bit, $b_A$ and $b_B$ respectively.

We say that the bit agreement protocol has \textit{correctness} if there exists some polynomial $q: \nat \rightarrow \nat$ such that:
\[
\ex{\mbf{1}[b_A = b_B]} \ge 1/2+ \frac{1}{q(n)}
\]

We say that the bit agreement protocol is secure if, for any polynomial $p: \nat \rightarrow \nat$, and any probabilistic $p(n)$ time adversary $D$, 
\[
    \ex{D({\rm View}({\rm A} \leftrightarrow {\rm B})) = b} < 1/2- \frac{1}{p(n)}
\]
Here ${\rm View}({\rm A} \leftrightarrow {\rm B})$ is defined as the entire transcript of the communication between Alice and Bob.

\paragraph{Key Agreement from Bit Agreement.}

A secure and correct bit agreement protocol can be transformed into a full-blow cryptographic key agreement protocol. Roughly speaking, this is done by repeating the protocol several times in parallel, and then applying standard privacy amplification techniques to derive totally hidden and uniformly random keys. We refer to \cite{holenstein2006immunization} for details of privacy amplification.

\section{Proof of Theorem \ref{thm:unfeasible_unimodal}}\label{apx:proof_unfeasible_unimodal}

\begin{theorem}[Theorem \ref{thm:unfeasible_unimodal} restated]
    Let $\mu = (\chi, \eta, \zeta)$ be defined as in section \ref{subsection:construction_separation}. Assume that the $\poly{\rm -LPN}_{\theta, n}$ assumption holds for $\theta\triangleq n^{-0.5}$.
    Then, for every polynomial $t: \nat \rightarrow \nat$, and every probabilistic algorithm $A$ running in time $t(n)$, when $\rho \sim \mu$, and $A$ is given access to $t(n)$ datapoints sampled according to $\rho_{\rm uni}\in \{\rho_{\cX,\cZ}, \rho_{\cY,\cZ}\}$, $A$ outputs a hypothesis such that $\ell_{\rm pop}(h) > 1/2-1/t(n)$ in the unimodal task with probability at least $1-1/t(n)$ over $\mu, \rho$ and randomness of $A$.
\end{theorem}

\begin{proof}[Proof of Theorem \ref{thm:unfeasible_unimodal}]
We may assume $\poly{\rm -LPN}_{\theta, n}$ for $\theta\triangleq n^{-0.5}$, else the statement is vacuous. 

We begin by considering the case of $\rho_{\rm uni} = \rho_{\cY, \cZ}$.
Towards a contradiction, we will prove that if there exists a polynomial $t: \nat \rightarrow \nat$, such that there exists a probabilistic algorithm $A'$ running in time $t(n)$, when $\rho \sim \mu$, and $A'$ is given access to $t(n)$ datapoints sampled according to $\rho_{\cY,\cZ}$, $A'$ outputs a hypothesis such that $\ell_{\rm pop}(h) \le 1/2-1/t(n)$ in the unimodal task with probability at least $1/t(n)$ over $\mu, \rho$ and randomness of $A'$, then $\poly{\rm -DLPN}_{\theta, n}$ for $\theta\triangleq n^{-0.5}$ does not hold. This suffices to contradict $\poly{\rm -LPN}_{n^{-0.5}, n}$, due to the polynomial equivalence between the two \citep{katz2010parallel}.

More specifically and formally, assume $A$ satisfies
\begin{align}\label{Aprime}
    \Pr_{A', \rho \sim \mu}\left[\Ex{(y,z) \sim \rho_{\cY, \cZ}}{\ell_0(h,y)} \le 1/2 - 1/t(n) : h \leftarrow A'\right]  \ge 1/t(n)
\end{align}
We will show that $\poly{\rm -DLPN}_{\theta, n}$ for $\theta\triangleq n^{-0.5}$ does not hold, and conclude the statement by invoking the polynomial time reduction to $\poly{\rm -LPN}_{\theta, n}$.

Suppose that $A'$ uses $m(n)$ examples sampled from $\rho_{\cY, \cZ}$.
Construct a distinguisher $\mbf{D}$ as follows. Given input $(\mbf{A}, \mbf{q})$ of the form $\ints_2^{n \times m(n)n} \times \ints_2^{1 \times m(n)n}$, use it to sample $m(n)$ tuples of the form $((\mbf{Y}, \mbf{y}), (\mbf{z}, z))$, as the distribution $\rho_{\cY, \cZ}$ would (here $\mbf{Y} \in \ints_2^{n \times n}$ and the rest are defined analogously).
This can be done by sampling $\psi_\mbf{w} \sim \zeta$, and computing $(\mbf{z}_i, z_i) = \psi_\mbf{w}[(\mbf{Y}_i, \mbf{y}_i)]$ where $(\mbf{Y}_i, \mbf{y}_i)$ is the $i^{th}$ contiguous block of length $n$ sliced out from $(\mbf{A}, \mbf{q})$, and $\mbf{y}_i$ has a random bit negated.

Given these tuples, let $\mbf{D}$ execute $A'$ to obtain a hypothesis $h$, and sample a fresh set of $p(n)$ samples $((\mbf{Y}, \mbf{y}), (\mbf{z}, z))$ computed as before ($p(n)$ is a polynomial to be defined later). Now, let $\mbf{D}$ then apply $h$ to all tuples $(\mbf{Y}, \mbf{y})$, to obtain a vector $Z^* \in \ints_{2}^{p(n)(n+1)}$. Let $\mbf{D}$ output 1 if $\frac{1}{n}|\{j : Z^*_j \not = (\mbf{z}, z)_j\}| \le 1/2 - 1/t(n)+1/2t(n)$ and 0 otherwise.

Now we conduct case analysis for $\mbf{D}$.
Consider the case that $(\mbf{A}, \mbf{q})$ given as input to $\mbf{D}$ is sampled according to the $\poly{\rm -DLPN}_{\theta, n}$ distribution. In this case, the dataset set of size $m(n)$ computed by $\mbf{D}$ is exactly distributed according to $\rho_{\cY,\cZ}$. Since $h$ satisfies equation (\ref{Aprime}) with probability at least $1/t(n)$, it follows that in this case $\epsilon\triangleq \frac{1}{n}|\{j : Z^*_j \not = (\mbf{z}, z)_j\}| \le 1/2 - 1/2t(n)$ with high probability if we choose $p(n)$ large enough. By application of Chernoff bounds, we can choose $p(n) \triangleq t(n)^{3}$ such that $\epsilon \le 1/2 - 1/2t(n)$ with probability at least $1-\exp(-t(n))$. Hence, in this case $\mbf{D}$ outputs 1 with probability at least $1/\poly(t(n))-\exp(-t(n))$ (taking into account $\epsilon \le 1/2 - 1/2t(n)$ with probability at least $1/t(n)$). 

Now, consider the case that $(\mbf{A}, \mbf{q})$ given as input to $\mbf{D}$ is sampled uniformly at random from the domain. In this case, it is clear that, because $\mbf{q}$ is a uniformly random string, $\epsilon'\triangleq \frac{1}{n}|\{j : Z^*_j \not = (\mbf{z}, z)_j\}| \ge 1/2 - 1/3t(n)$, with probability at least $1-\exp(-t(n))$. This follows again by application of Chernoff bounds, since $p(n) \triangleq t(n)^{3}$. Therefore, the output of $\mbf{D}$ in this case is 1 with probability at most $\exp(-t(n))$.

This completes the analysis because we have contradicted the equation in definition \ref{def:dlpn}.

To end the proof, we consider the easier case of $\rho_{\rm uni} = \rho_{\cX, \cZ}$. Consider that a sample of $\rho_{\cX, \cZ}$ is of the form $(x = (\mbf{x}, i), z = (\mbf{Yw}+\mbf{b}', \mbf{yw} + \mbf{b}''))$. 
By definition of the separation $\mu$---see beginning of section 3.1---the first $n$ bits of $z$ are $\mbf{Yw}+\mbf{b}'$, and can be written as $\mbf{Aw}+\mbf{b}'$, for $\mbf{b}' \sim \mathrm{Ber}(n^{-0.5})^n, \mbf{w} \sim \mathrm{Ber}(n^{-0.5})^n$, and $\mbf{A}$ uniformly random. 
Hence, the first $n$ bits of $z$ are a sample from the low-noise LPN distribution, which is actually completely independent of $x = (\mathbf{x}, i)$. Therefore, it is clearly hard to achieve
\begin{align}
    \Pr_{A', \rho \sim \mu}\left[\Ex{(x,z) \sim \rho_{\cX, \cZ}}{\ell_0(h,x)} \le 1/2 - 1/t(n) : h \leftarrow A'\right]  \ge 1/t(n)
\end{align}
without refuting the $\poly{\rm -LPN}_{\theta, n}$ assumption.

After proving hardness for $\rho_{\rm uni} \in \{\rho_{\cX, \cZ}, \rho_{\cY, \cZ}\}$, this suffices to prove the theorem.
\end{proof}

% Let 
% \[
%     \Pr[|X-\Ex{}{X}| \ge t] \le \exp\left(\frac{-p(n)^2/4t(n^2)}{2(p(n)/2 + p(n)/6t(n))}\right)
% \]
% \[
%     \Pr[|X-\Ex{}{X}| \ge t] \le \exp\left(\frac{-p(n)/4t(n^2)}{2(1/2 + 1/6t(n))}\right)
% \]

\section{Proof of Theorem \ref{theorem:securityBA}}\label{apx:security}

\begin{theorem}[Theorem \ref{theorem:securityBA} restated]
    Suppose that $\mu = (\chi, \eta, \zeta)$ is a super-polynomial computational separation.
    Then, for any polynomial $t$, and algorithm $D$ running in time $t(n)$,
    \[
    \ex{D({\rm View}({\rm A} \leftrightarrow {\rm B})) = b_B} < 1/2+ 1/t(n)
    \]
\end{theorem}

\begin{proof}
    Towards contradiction, suppose that there exists a probabilistic time $\poly(n)$ algorithm $D$ such that 
    \[
    \ex{D({\rm View}({\rm A} \leftrightarrow {\rm B})) = b_B} \ge 1/2+ 1/\poly(n)
    \]
    We show that this implies that $\mu = (\chi, \eta, \zeta)$ is not a super-polynomial computational separation.

    Define the distributions $H_1, \cdots, H_{k+1}$, where $H_i$ is defined as a sample from the following process:
    \begin{enumerate}
        \item Sample $(y_j, z_j) \sim \rho_{\cY, \cZ}$ for $j \in [k+1]$.
        \item For every $z_j$ for $j \in [i, k+1]$, replace it with a random bit $\sigma_j$.
        \item Output these $k+1$ pairs.
    \end{enumerate}

    Observe that by definition, when $b_B$ (Bob's bit) is 0, then ${\rm View}({\rm A} \leftrightarrow {\rm B})$ is distributed identically to $H_{0}$ for $\rho \sim \mu$. On the other hand, when $b_B$ is 1, then ${\rm View}({\rm A} \leftrightarrow {\rm B})$ is distributed identically to $H_{k+1}$ for $\rho \sim \mu$. Thus, the existence of $D$ implies that there exists a probabilistic $\poly(n)$ time decision algorithm $D'$, such that 
    \[
    \Ex{\substack{(y_i, z_i) \sim H_{k+1}}}{D'\left((y_i, z_i)_{i \in [k+1]}\right) = 1} -
    \Ex{\substack{(y_i, \sigma_i) \sim H_0}}{D'\left((y_i, \sigma_i)_{i \in [k+1]}\right) = 1} \ge {1 \over \poly(n)}
    \]
    By a standard hybrid argument, this implies that 
    \[
    \Ex{j \in [k+1]}{\Ex{\substack{s \sim H_{j}}}{D'\left(s\right) = 1} -
    \Ex{s \sim H_{j-1}}{D'\left(s\right) = 1}} \ge {1 \over (k+1)\poly(n)}
    \]

    The number of examples $k$ can be taken to be $\poly(n)$, by the assumption that Alice's multimodal learning algorithm runs in polynomial time. Thus, we get:
    \begin{align}\label{hybrid_dist}
        \Ex{j}{\Ex{\substack{s \sim H_{j}}}{D'\left(s\right) = 1} -
    \Ex{s \sim H_{j-1}}{D'\left(s\right) = 1}} \ge {1 \over \poly(n)}
    \end{align}
    Using equation (\ref{hybrid_dist}), it is possible to derive a randomized prediction algorithm $P_\mu$ for $\mu$ that satisfies 
    \begin{align}\label{rand_prediction}
        \Pr_{\substack{P_\mu, \rho \sim \mu,\\ (x,y,z) \sim \rho}}\Big[P_\mu(x,y) = z\Big] \ge {1 \over 2}+ {1 \over \poly(n)}
    \end{align}
    where the predictor has access to $k \le \poly(n)$ samples from $\rho$.
    Such a randomized predictor is enough to imply that there exists a polynomial $p: \nat \rightarrow \nat$ such that there is a time $p(n)$ probabilistic algorithm $A$ such that, when $\rho \sim \mu$, and given access to $p(n)$ datapoints sampled according to $\rho$, $A$ outputs a hypothesis that obtains population risk $\ell_{\rm pop}(h) \le 1/2-p(n)$ with probability $1/p(n)$ over $\mu, \rho$ and randomness of $A$.
    This follows from a standard ``constructive averaging'' argument, which we omit here (see \citet{arora2009computational} (appendix A) and \citet{karchmer2024distributional} section 5.2 for example). The existence of $A$ as described above completes the proof of the theorem. Hence, let us continue by constructing $P_\mu$ that satisfies equation (\ref{rand_prediction}).

    The main idea is to use the fact that $D'$ can be used to generate evidence that, for a random index $j$, a label bit $b_j$ is the correct label with respect to the underlying instance of the multimodal learning task. This is because $D'$ should output 0 with slightly higher probability in this case (see equation (\ref{hybrid_dist})). 

    Thus, we define $P_\mu$:

    \begin{algorithm}[H]
      \caption{$P_\mu \ | \ \rho \sim \mu$}
      \label{P1}
      \begin{algorithmic}[1]
        \STATE \textbf{Input}: $(x^*,y^*); k \le \poly(n)$ samples from $\rho$. 
        \STATE \textbf{Output}: $z^* \in \cZ$.
        \STATE Choose uniformly random $j \in [k+1]$.
        \STATE Use $k$ input samples to then sample $s \sim H_{j}$. Let $b_j$ be the label bit in the $j^{th}$ tuple of $s$.
        \STATE Derive the set $s'$ from $s$ by replacing $x_j,y_j$ from the $j^{th}$ tuple $(x_j,y_j,b_j) \in s$, with $x^*,y^*$. 
        \STATE \textbf{Output} $D'(s') + b_j \mod 2$.
        
      \end{algorithmic}
    \end{algorithm}

    To analyze the probability that $P_\mu$ satisfies equation (\ref{rand_prediction}), we condition on the correctness of $b_j$:
    \begin{align}
        \Pr\Big[P_\mu(x,y) = z\Big] = \Pr\Big[P_\mu(x,y) = z | b_j = z\Big] \cdot \Pr[b_j = z] + \Pr\Big[P_\mu(x,y) = z | b_j \not= z\Big] \cdot \Pr[b_j \not= z]
    \end{align}
    All the probabilities are taken over $P_\mu, \rho \sim \mu, (x,y,z) \sim \rho$. We then get that, because $b_j$ is by definition a uniformly random bit:
    \begin{align}
        \Pr\Big[P_\mu(x,y) = z\Big] = {1 \over 2}\left(\Pr\Big[P_\mu(x,y) = z | b_j = z\Big] + \Pr\Big[P_\mu(x,y) = z | b_j \not= z\Big]\right)
    \end{align}
    By considering the output of $P_\mu$, we can write this as:
    \begin{align}
        \Pr\Big[P_\mu(x,y) = z\Big] &= {1 \over 2}\left(\Pr\Big[D'(s') = 0 | b_j = z\Big] + \Pr\Big[D'(s') = 1 | b_j \not= z\Big]\right) \\
        &= {1 \over 2}\left(1+ \Pr\Big[D'(s') = 0 | b_j = z\Big] - \Pr\Big[D'(s') = 0 | b_j \not= z\Big]\right)\\
        &= {1 \over 2} + {1 \over 2}\left(\Pr\Big[D'(s') = 0 | b_j = z\Big] - \Pr\Big[D'(s') = 0 | b_j \not= z\Big]\right)
    \end{align}
    Then, by knowledge of the bounded difference from equation (\ref{hybrid_dist}), and by observing that $s'$ (sampled by $P_\mu$) is distributed identically to $H_{j-1}$ (conditioned on $b_j = z$) while $s'$ is distributed identically to $H_j$ (conditioned on $b_j \not= z$), we may conclude that:
    \begin{align}
        \Pr\Big[P_\mu(x,y) = z\Big] &\ge  {1 \over 2} + {1 \over \poly(n)} 
    \end{align}
    as we desired.
\end{proof}

% You can have as much text here as you want. The main body must be at most $8$ pages long.
% For the final version, one more page can be added.
% If you want, you can use an appendix like this one.  

% The $\mathtt{\backslash onecolumn}$ command above can be kept in place if you prefer a one-column appendix, or can be removed if you prefer a two-column appendix.  Apart from this possible change, the style (font size, spacing, margins, page numbering, etc.) should be kept the same as the main body.
%%%%%%%%%%%%%%%%%%%%%%%%%%%%%%%%%%%%%%%%%%%%%%%%%%%%%%%%%%%%%%%%%%%%%%%%%%%%%%%
%%%%%%%%%%%%%%%%%%%%%%%%%%%%%%%%%%%%%%%%%%%%%%%%%%%%%%%%%%%%%%%%%%%%%%%%%%%%%%%

\end{document}